\documentclass[conference]{IEEEtran}
\IEEEoverridecommandlockouts
\usepackage{cite}
\usepackage{amsmath,amssymb,amsfonts}
\usepackage{algorithmic}
\usepackage{graphicx}
\usepackage{textcomp}
\usepackage{xcolor}
\usepackage{subfigure}
\def\BibTeX{{\rm B\kern-.05em{\sc i\kern-.025em b}\kern-.08em
    T\kern-.1667em\lower.7ex\hbox{E}\kern-.125emX}}
\newtheorem{mypro}{Proposition}
\newtheorem{proof}{Proof}
\begin{document}

\title{Segmenting Epipolar Line\\
}

\author{\IEEEauthorblockN{Shengjie Li}
	\IEEEauthorblockA{School of Electronic Information and\\Electrical Engineering\\Shanghai Jiao Tong University\\
		shengjie\_li@sjtu.edu.cn}
	\and
	\IEEEauthorblockN{Qi Cai}
	\IEEEauthorblockA{School of Electronic Information and\\Electrical Engineering\\Shanghai Jiao Tong University\\
		sipangyiyou@sjtu.edu.cn}
	\and
	\IEEEauthorblockN{Yuanxin Wu}
	\IEEEauthorblockA{School of Electronic Information and\\Electrical Engineering\\Shanghai Jiao Tong University\\
		yuanxin.wu@sjtu.edu.cn}
	}

\maketitle

\begin{abstract}
 Identifying feature correspondence between two images is a fundamental procedure in three-dimensional computer vision. Usually the feature search space is confined by the epipolar line. Using the cheirality constraint, this paper finds that the feature search space can be restrained to one of two or three segments of the epipolar line that are defined by the epipole and a so-called virtual infinity point.
\end{abstract}


\section{Introduction}
Now photography has become an indispensable part of people's daily life, using the camera can help people record things in life. However, the real world is three-dimensional, and ordinary monocular camera can only record the world on a two-dimensional image by projection. Then how to use two-dimensional images to get three-dimensional information has become a very concerned problem. This involves the problem of image matching. If we can know the two-dimensional image of an object from different angles, we can find a way to recover the three-dimensional information from it. Therefore, the correspondence of two image points is a fundamental problem in computer vision. For two images without a priori information we usually use feature detection methods like SIFT\cite{lowe2004distinctive} or SURF \cite{bay2006surf} to find the correspondence between images. In general, only a few points on the image are suitable for matching, e.g., corners or edge points, and the feature matching methods are time consuming. If the matching relationship between the photos taken by two cameras is known, we can get the essential matrix and the relative pose of two cameras\cite{hartley2003multiple,nister2004efficient,stewenius2006recent,longuet1981computer}. And the epipolar constraint can be used to reduce time consumption in feature matching. The epipolar constraint reduces the search space of matching points from two dimensions to one dimension. That is to say, given one image point in one view, the corresponding image point in the other view lies on the intersection of the imaging plane of the other view and the epipolar plane passing the two view origins and the image point, namely the epipolar line \cite{hartley2003multiple}.
The methods to speed up the epipolar line search have been studied for many years. We may use image rectification to make the epipolar lines all parallel so that the search can be done on scan lines \cite{fusiello2000compact,oram2001rectification,pollefeys1999simple}, as for rectified images the disparity can be used to represent the depth of a pixel\cite{qian1997binocular,gonzalez1998neural}. However, the technique of image rectification distorts the image and loses surface details on object reconstruction \cite{blumenthal2014high} and can not be used under some circumstances. For example, when the epipoles are in the image, it is impossible to perform image rectification. Because when doing image rectification, the epipoles need to be transformed to infinity through affine transformation. If the epipoles are in the image, the points around the epipoles are also transformed to infinity, and the original image will be destroyed. Our method does not need the process of image rectification, so there are no such limitations. Our method combines epipolar line search and the cheirality constraint. The cheirality constraint, first proposed by Hartley in \cite{10.1023/A:1007984508483}, means that any point that lies in an image must lie in front of the camera producing that image, which is alternatively known as the positive depth constraint. Werner and Pajdla \cite{werner2001cheirality} give necessary and sufficient conditions for an image point set to correspond to any real imaging geometry. Agarwal and Pryhuber\cite{agarwal2020multiview} give an algebraic description of mutiview cheirality. In this paper, we use the cheirality constraint to segment epipolar line and identify the correct epipolar segment.

This paper is organized as follows: Section 2 reviews the existing concepts such as the epipolar geometry and cheirality constraint. Section 3 raises the concepts of virtual infinity point and uses it for segmenting the epipolar line. In section 3, cheirality on the epipolar line is discussed. Section 4 describes how to identify the epipolar line segment on which the corresponding image point lies. Section 5 reports the experiment results to show the effectiveness of our new method. Section 6 concludes this paper.

\section{Related Works}
Throughout this paper we assume that all cameras are calibrated and the normalized image coordinates are used unless explicitly stated otherwise.
\subsection{Epipolar Geometry}

The epipolar geometry between two cameras is illustrated in Fig.~\ref{epipolarsegment}. $X=[x,y,1]^T$ and $X'=[x',y',1]^T$ are a pair of matched image points. $X_w=[x_w,y_w,z_w,1]^T$ is the projected 3D world point relative to the first camera coordinate. $X^{'}_w=[x^{'}_w,y^{'}_w,z^{'}_w,1]^T$ is the coordinate of projected 3D world point relative to the second camera. The two-view imaging can be described as 

\begin{align}
z_w^{'}X^{'}=z_{w}RX+t
\label{epipolarconstraint}
\end{align}
where $R$ and $t$ are the rotation and translation between two camera frames, respectively.
Left multiplying \eqref{epipolarconstraint} by $X^{'T}[t]_{\times}$ yields the epipolar constraint \cite{hartley2003multiple}
\begin{align}
X^{'T}EX=0
\label{essentialmatrix}
\end{align}
where the essential matrix $E=[t]_{\times}R$ and $[t]_{\times}$ means the skew-symmetric matrix formed by $t$.

\begin{figure}[htbp]
	\centerline{\includegraphics[width=6cm,height=4.5cm]{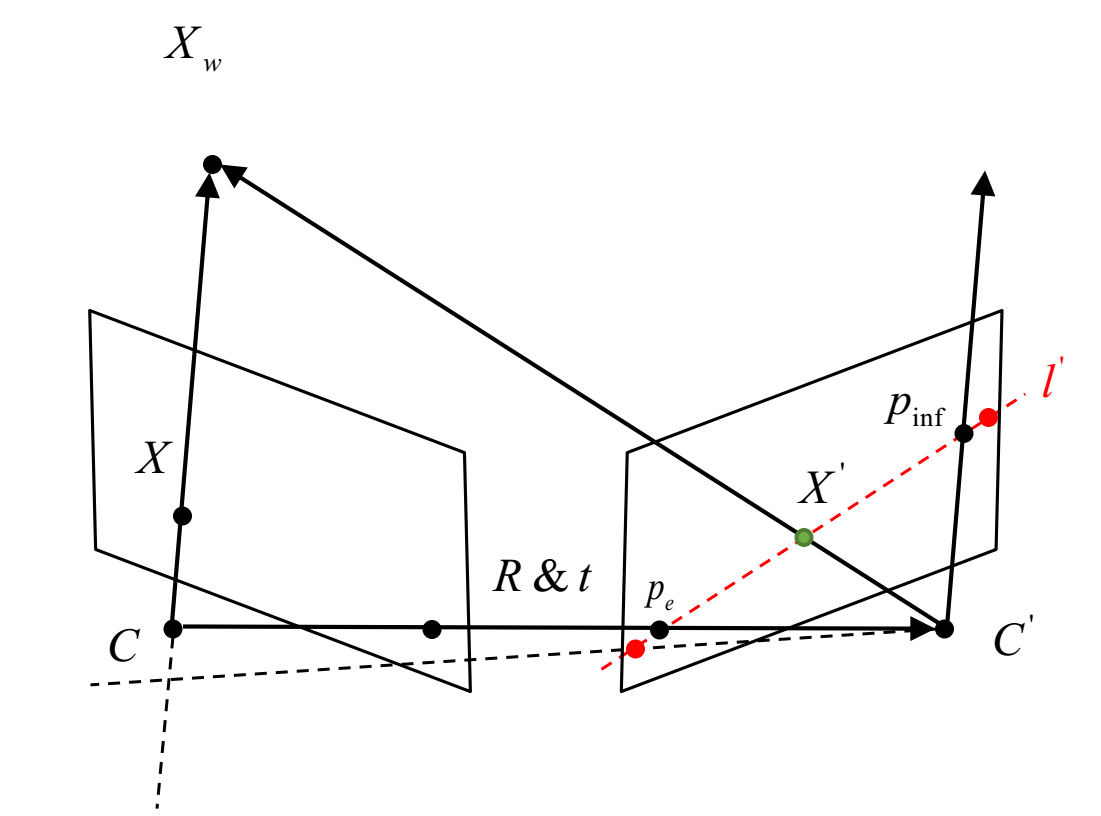}}
	\caption{An example of epipolar line, virtual infinity point and epipole in two-view geometry.}
	\label{epipolarsegment}
\end{figure}
\subsection{Epipolar Line}
Referring to \eqref{essentialmatrix}, $X^{'T}$ is restricted to fall in the left null space of $EX$ which can be described as a line on the image, namely the epipolar line. The line can be described as 
\begin{align}
l^{'}=EX=[t]_{\times}RX
\label{epipolarline}
\end{align}
In order to find a pair of matching points between two images, the epipolar line can be employed to reduce the search space from two dimensions to one dimension. It is known as the epipolar line search \cite{hartley2003multiple}. 
\subsection{Cheirality Constraint}
The cheirality constraint in two-view geometry means that a 3D world point must lie in front of both two cameras, or alternatively it must have positive depths. It is well known that we would get four sets of pose solutions from decomposing the essential matrix $E$ \cite{hartley2003multiple,wang2000svd}. The cheirality constraint is helpful in identifying the right pose solution \cite{hartley2003multiple,10.1023/A:1007984508483}.

\section{Epipolar Line  Segments}
This section will show that the epipolar line has two or three segments that are divided by the epipole and a so-called virtual infinity point.
\subsection{Virtual Infinity Point}
The virtual infinity point $p_{inf}$ is defined as 
\begin{align}
p_{inf}=\frac{RX}{[RX]_{3}}
\label{infpts}
\end{align}
where $[RX]_3$ is the third row of $RX$. In other words, the virtual infinity point $p_{inf}$ is $RX$ represented by the homogeneous coordinate.
\begin{mypro}
	The virtual infinity point lies on the epipolar line.
\end{mypro}
\begin{proof} with \eqref{epipolarline} and \eqref{infpts}, we have
	\begin{align}
	p_{inf}\times l^{'}=(\frac{RX}{[RX]_{3}})\times ([t]_{\times}RX)=0
	\label{infonline}
	\end{align}
	Q.E.D.
\end{proof}
The epipole $p_e$ on the right image is computed by \cite{hartley2003multiple}
\begin{align}
p_e=\frac{t}{[t]_{3}}
\label{epipole}
\end{align}
where $[t]_3$ is the third row of $t$.

Specifically, $p_{inf}$ is at infinity when $[RX]_{3}=0$ and 
$p_e$ is at infinity when $[t]_{3}=0$.

Referring to Fig.~\ref{epipolarsegment}, $C$ and $C^{'}$ are two camera centers. The projection rays $\overrightarrow{CX}$ and $\overrightarrow{C'p_{inf}}$ are parallel and intersect at the plane at infinity. It is the reason we name $p_{inf}$ as the virtual infinity point. 

\subsection{Segmentation of Epipolar Line}
It is well known that the epipolar line passes through the epipole. In Sec. 3.A, we prove that the virtual infinity point lies on the epipolar line as well. Therefore, the epipolar line are separated into three segments by the virtual infinity point and the epipole, as demonstrated in Fig.~\ref{epipolarsegment}. Without loss of generality, it assumes that $p_{inf}$ lies on the right side of $p_e$. It can be inferred that the correct matching point should lie on one of the three segments for finite epipole and finite virtual infinity point. The epipolar line would be separated into two segments for the special case $[RX]_{3}=0$ or  $[t]_{3}=0$.
\subsection{Cheirality on Epipolar Line}
Referring to Fig.~\ref{epipolarsegment}, we assume $\widetilde{X'}$ represents one candidate on the epipolar line and $\widetilde{X_w}$ represents the intersection of $\overrightarrow{CX}$ and $\overrightarrow{C'\widetilde{X'}}$. As illustrated by the green dot, if $\widetilde{X_w}$ is in front of both cameras, $\widetilde{X'}$ can be a possible solution to the matching point of $X$. As illustrated by the red dots, if $\widetilde{X_w}$ is behind one of the two cameras, the cheirality constraint is violated and then $\widetilde{X'}$ is not a feasible solution. That means that the cheirality constraint helps us remove infeasible solutions to the matching point on the condition that the depth can be obtained.

\section{Identifying Right Epipolar Segment}
Referring to \eqref{epipolarconstraint}, $X^{'}$ can be represented as 
\begin{align}
X'=\lambda_1 RX+\lambda_2 t
\label{PPO2}
\end{align}
where $\lambda_1=\frac{z_{w}}{z^{'}_{w}}$  and $\lambda_2=\frac{1}{z^{'}_{w}}$. It is well known that if $z_{w}>0$,$z^{'}_{w}>0$ in \eqref{epipolarconstraint}, the image points pair $(X,X^{'})$ satisfy with the cheirality constraint. If $\lambda_1>0$,$\lambda_2>0$, we have $z_{w}>0$,$z^{'}_{w}>0$. Therefore \eqref{PPO2} constrains any point pair to fit the real imaging geometry when $\lambda_1>0$ and  $\lambda_2>0$. 
The relative position on epipolar line of correct matching point $X'$, virtual infinity point $p_{inf}$ and epipole $p_{e}$ is decided by the sign of $[RX]_3$ and $[t]_3$.
\begin{mypro}
	
	\centerline{$[RX]_3>0$,$[t]_3>0$, $X'$ is between $p_{inf}$ and $p_e$}
	
	\centerline{$[RX]_3>0$,$[t]_3<0$, $p_{inf}$ is between $X'$ and $p_e$}
	
	\centerline{$[RX]_3<0$,$[t]_3>0$, $p_e$ is between $p_{inf}$ and $X'$}
	\label{pro2}
\end{mypro}
\begin{proof}
	Substituting \eqref{infpts} and \eqref{epipole} into \eqref{PPO2} gives
	\begin{equation}
	\begin{aligned}
	X'&=\lambda_1 [RX]_{3}\frac{RX}{[RX]_3}+\lambda_2 [t]_3\frac{t}{[t]_3} \\
	&=\lambda_1 [RX]_{3}p_{inf}+\lambda_2[t]_{3}p_{e} \\
	&=ap_{inf}+bp_e  
	\end{aligned}
	\label{ab3}
	\end{equation}
	
	where $a=\lambda_1[RX]_{3}$, $b=\lambda_2[t]_3$. Because $X^{'},p_{inf}$ and $p_e$ are represented in homogeneous coordinates, we have $a+b=1$. Specifically, if $\lambda_1=0$, then $b=1,X^{'}=p_e$, and if $\lambda_2=0$, $a=1,X^{'}=p_{inf}$. For $\lambda_1>0,\lambda_2>0$, $a$ and $b$ have the same signs with $[RX]_3$ and $[t]_3$, respectively. Then \eqref{ab3} can be transformed as such
\begin{gather}
(a+b)X^{'}=ap_{inf}+bp_e \\
a(X^{'}-p_{inf})=b(p_e-X^{'})
\label{PPOproof1}
\end{gather}  
When $a>0,b>0$, \eqref{PPOproof1} indicates that $\overrightarrow{p_{inf}X^{'}} \parallel \overrightarrow{X^{'}p_e}$. Considering that $X'$, $p_{inf}$ and $p_e$ is colinear, we can infer that $X'$ is between $p_{inf}$ and $p_e$. Furthermore, \eqref{ab3} can also be transformed into
\begin{gather}
X^{'}=(1-b)p_{inf}+bp_e \\
-b(p_e-p_{inf})=p_{inf}-X^{'}
\end{gather}
When $b<0$, it means $\overrightarrow{p_{inf}p_e} \parallel \overrightarrow{X^{'}p_{inf}}$, $p_{inf}$ is between $X'$ and $p_e$. Finally, \eqref{ab3} can also be transformed as such
\begin{gather}
X^{'}=ap_{inf}+(1-a)p_e \\
-a(p_{inf}-p_e)=p_e-X^{'}
\end{gather}
When $a<0$, we $\overrightarrow{p_ep_{inf}} \parallel \overrightarrow{X^{'}p_e}$, alternatively, $p_e$ is between $p_{inf}$ and $X'$.
Note that the condition $a<0,b<0$ does not occur because $a+b=1$.

Q.E.D.
\end{proof}

According to Sec. 3.B the epipolar line are separated into two segments for the special case $[RX]_{3}=0$ or $[t]_3=0$. Specifically, when $[RX]_3=0$ and $[t]_3\not=0$, \eqref{PPO2} can be transformed into 
\begin{equation}
\begin{aligned}
X'&=\lambda_1 RX+\lambda_2 [t]_3\frac{t}{[t]_3} \\
&=\lambda_1 RX+\lambda_2[t]_{3}p_{e} \\
&=\lambda_1 RX+bp_e  
\end{aligned}
\label{b}
\end{equation}
As $[X^{'}]_{3}=1,[RX]_{3}=0,[p_{e}]_{3}=1$, we have $b=1$. Therefore, $X^{'}-p_{e}=\lambda_1RX$, indicating the relative position of $X^{'}$ and $p_{e}$ on the epipolar line is only up to the first two rows of $RX$ ($\lambda_1>0$). For instance, if $[RX]_1>0$, $X^{'}$ lies to the right of $p_e$ in the x-axis direction of the image plane.
When $[t]_3=0$ and $[RX]_3\not=0$, \eqref{PPO2} can be transformed into
\begin{equation}
\begin{aligned}
X'&=\lambda_1 [RX]_{3}\frac{RX}{[RX]_3}+\lambda_2 t \\
&=\lambda_1 [RX]_{3}p_{inf}+\lambda_2 t \\
&=ap_{inf}+\lambda_2 t 
\end{aligned}
\label{a}
\end{equation} 
Because $[X^{'}]_{3}=1,[RX]_{3}=1,[p_{e}]_{3}=0$, we have $a=1$. Therefore, $X^{'}-p_{inf}=\lambda_2 t$ that means the relative position of $X^{'}$ and $p_{inf}$ on the epipolar line is only up to the first two rows of $t$  ($\lambda_2>0$).

When $[t]_3=0$ and $[RX]_3=0$, the third row of the right side of \eqref{PPO2} is equal to 0. As $X^{'}$ is represented in homogeneous coordinates, the left side of \eqref{PPO2} is equal to 1. Then the condition $[t]_3=0$ and $[RX]_3=0$ does not occur.

It should be noted that the actual epipolar line segment is the intersection of the segment we choose and the image plane in practice. Hartley \cite{hartley2003multiple} points out the search space can be reduced to a bounded line, but our method analytically proves that the correct
matching point can be located in any of the three segments. It is distinctive from the common sense that the correct matching point lies between the epipole and virtual infinity point.

\section{Experimental Results}
Firstly, we examine the theory of epipolar line segments on four pairs of images. The first and second pair have a small baseline. The third and fourth pair have a large baseline. The relative pose of images is already known and the SIFT features are used.

\subsection{Small-baseline test}
Five pairs of feature points are randomly chosen and highlighted as green dots in both images, as seen in Fig.~\ref{smallbaseline} and Fig.~\ref{smallbaseline1}. With the known relative pose, the epipolar line for each chosen point in the left image is plotted in the second image. All virtual infinity points lie on the corresponding epipolar lines. Blue rays are used to represent the right segment on the epipolar line and red rays are used to represent the wrong segments, identified according to Proposition~\ref{pro2}. We see that all the matching points in the second image fall on the right epipolar segment. It should be noted that the epipole in the second image is out of view. Therefore, the epipolar lines have only two visible segments and the correct matching points are close to the corresponding virtual infinity points. 
\begin{figure}[htbp]
	\centering
	\includegraphics[width=9cm,height=4.5cm]{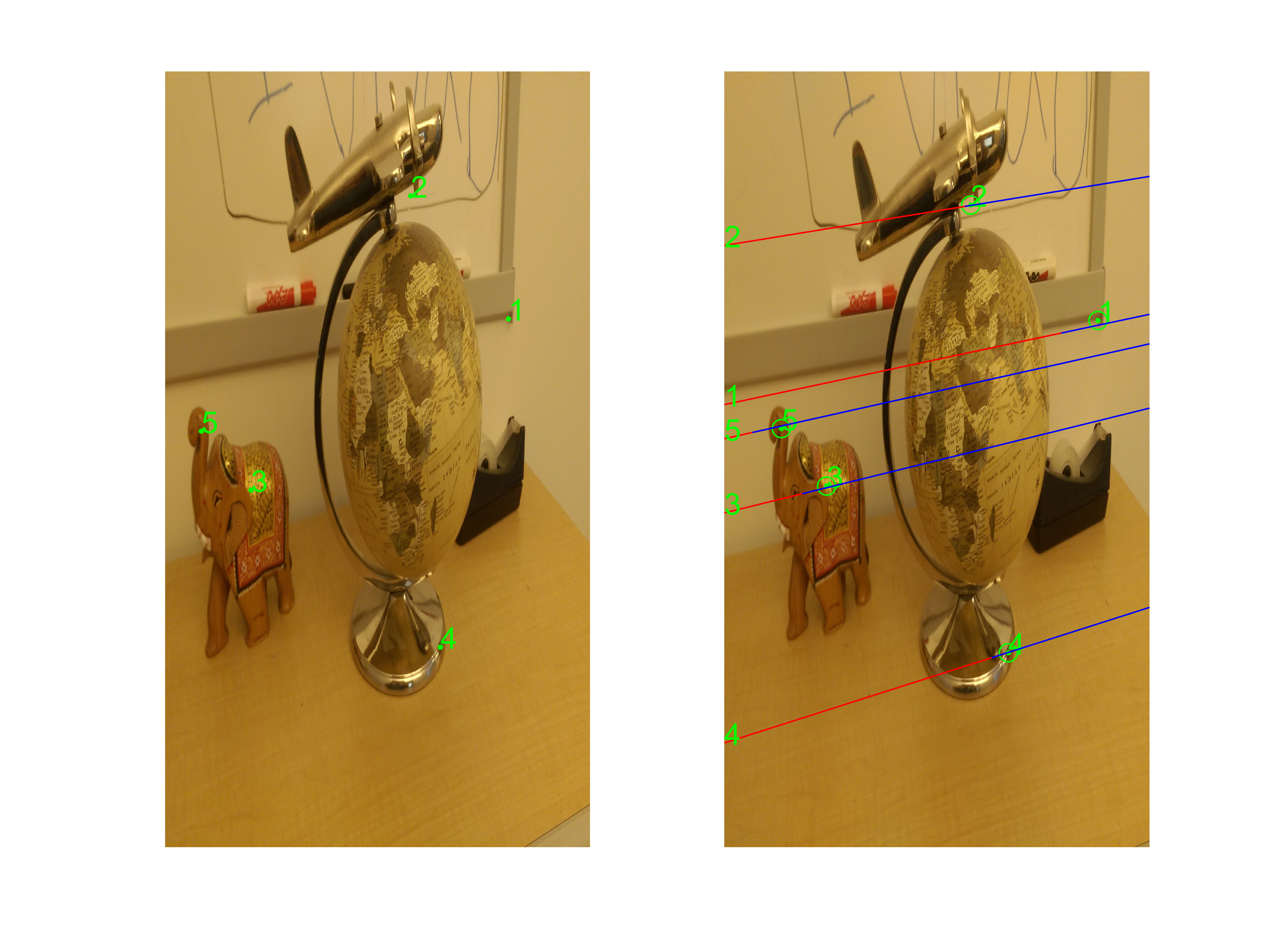}
	\caption{The first pair of images with small baseline}
	\label{smallbaseline}
\end{figure}

\begin{figure}[htbp]
	\centering
	\includegraphics[width=9cm,height=4.5cm]{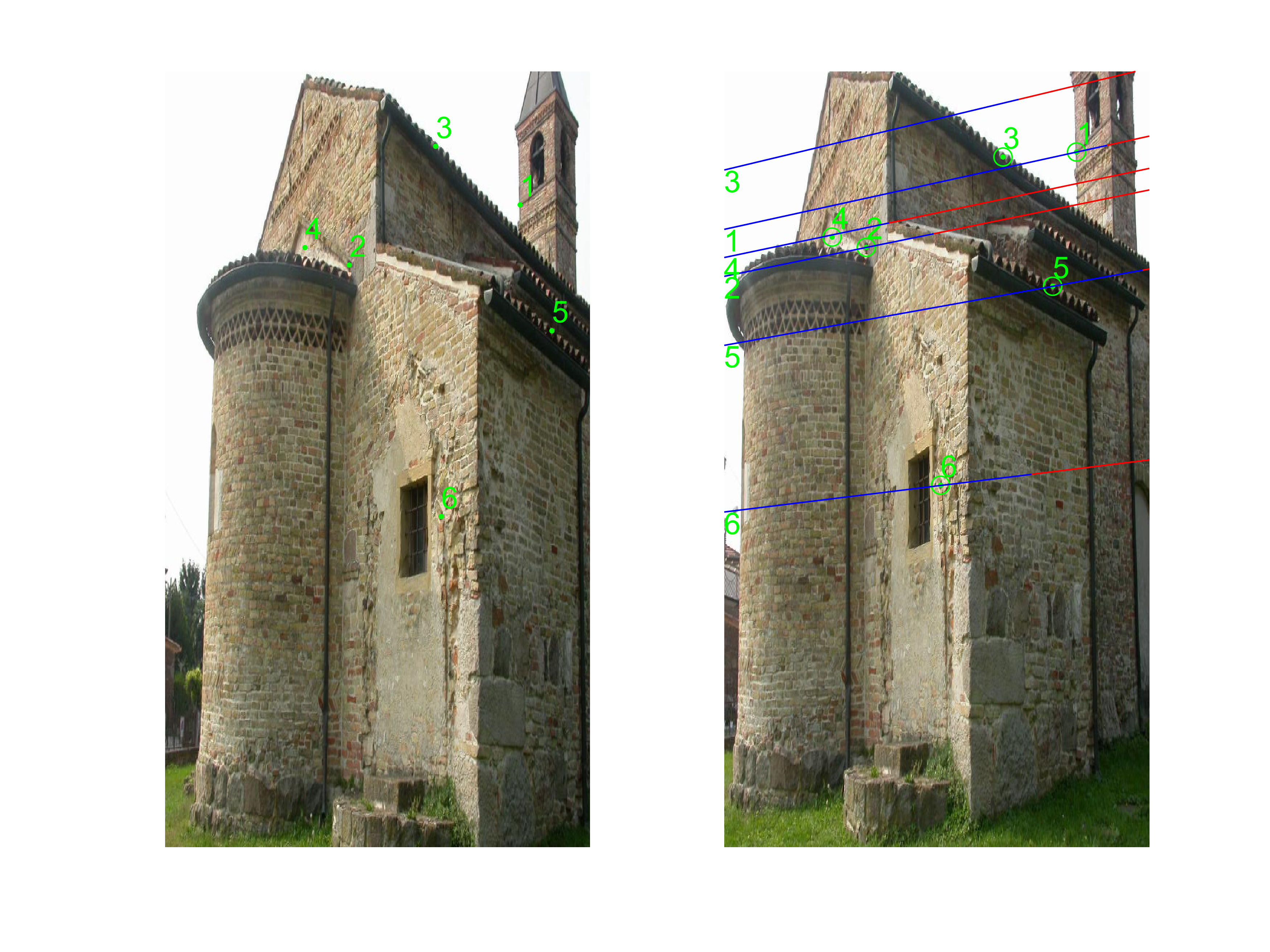}
	\caption{The second pair of images with small baseline}
	\label{smallbaseline1}
\end{figure}

\subsection{Large-baseline test}
In this case, the epipole in the second image falls in the image and all the epipolar lines intersect at the epipole as Fig.~\ref{largebaseline} illustrates. The epipole and virtual infinity points divide the epipolar lines into three segments and all correct matching points fall on the identified segments. Note that point $\#1$ and point $\#3$ in the right image are clearly false matches and stay far away from the epipolar lines of point $\#1$ and point $\#3$ in the left image. In Fig.~\ref{largebaseline1}, the epipole in the second image is out of view. We see that all the matching points in the second image fall on the right epipolar segment. Note that point $\#5$ in the right image are clearly false matches and stay far away from the epipolar lines of point $\#5$ in the left image.
\begin{figure}
	\centering
	\includegraphics[width=9cm,height=4.5cm]{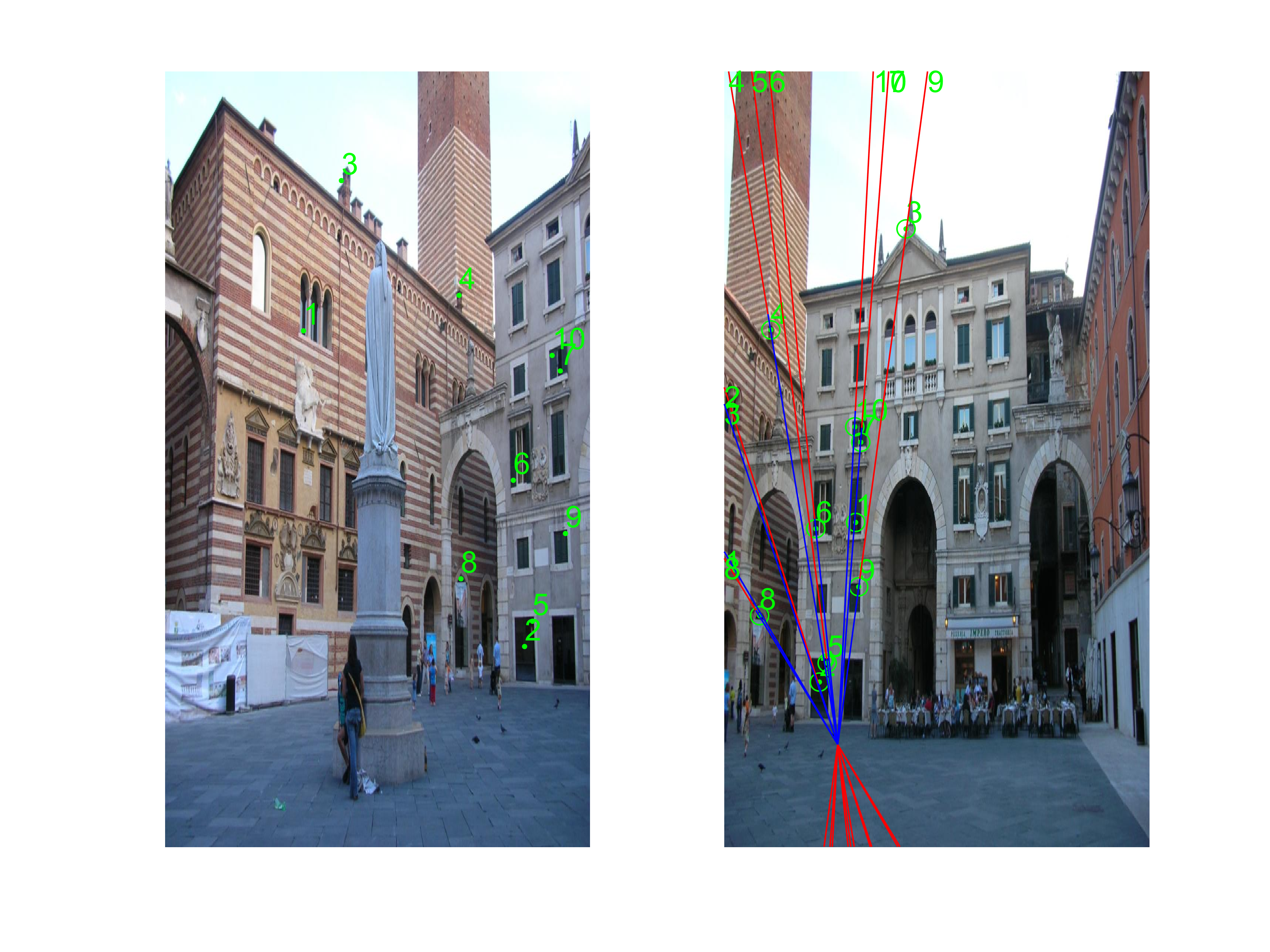}
	\caption{The first pair of images with large baseline}
	\label{largebaseline}
\end{figure}

\begin{figure}
	\centering
	\includegraphics[width=9cm,height=4.5cm]{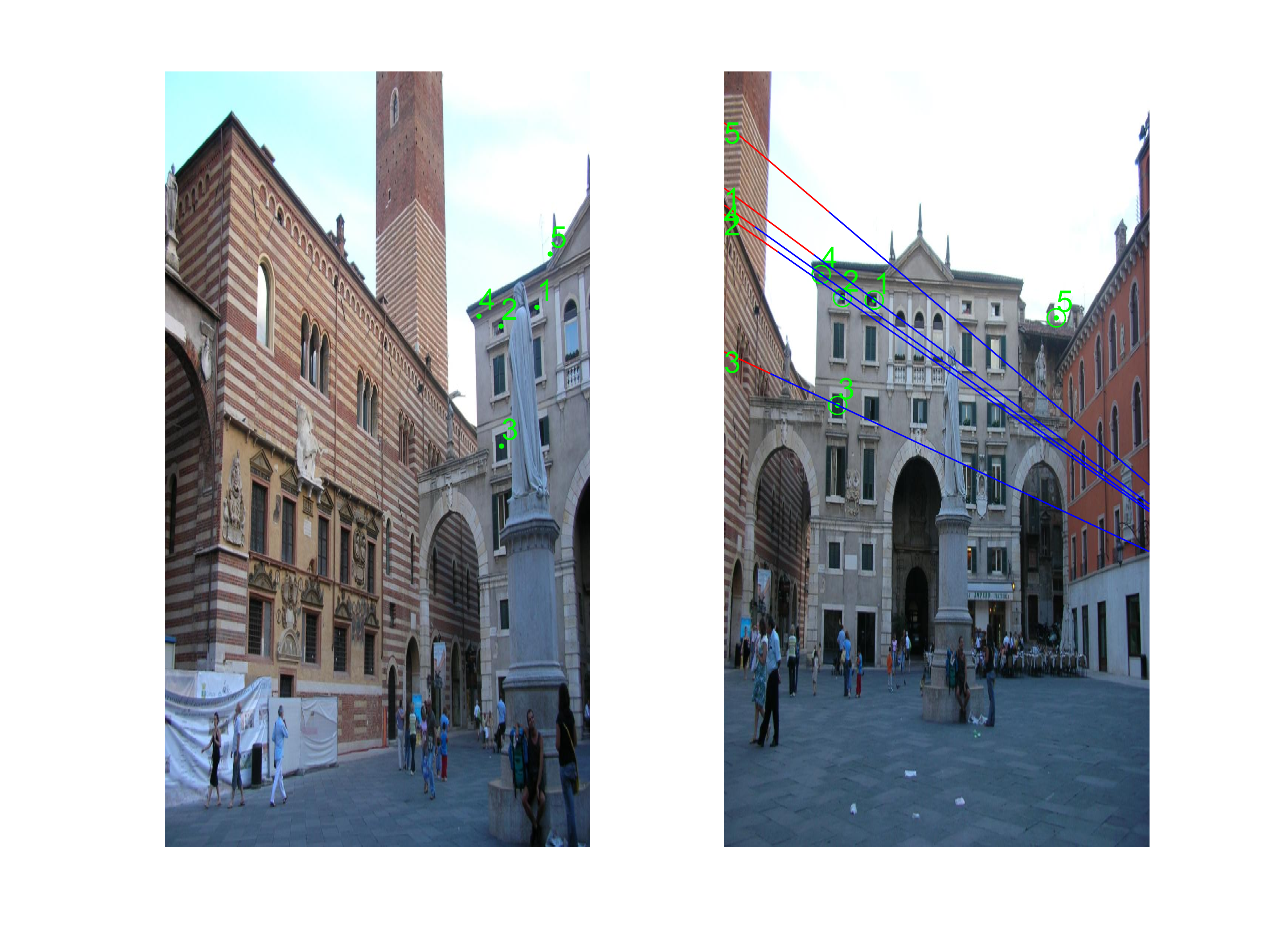}
	\caption{The second pair of images with large baseline}
	\label{largebaseline1}
\end{figure}
\subsection{Simulation test}
Additionally, we first simulate two-view geometry in different poses and calculate how much search space on the epipolar line can be reduced. The size of simulated image is set to $100\times100$ pixels. The intrinsic matrix of the two simulated cameras are both set to  
$
\left[
\begin{matrix}
50 & 0 & 50 \\
0 & 50 & 50 \\
0 & 0 & 1
\end{matrix}
\right] 
$.

\begin{figure}[htbp]
	\centering
	\subfigure[Adjacent cameras]{
		\includegraphics[width=9cm]{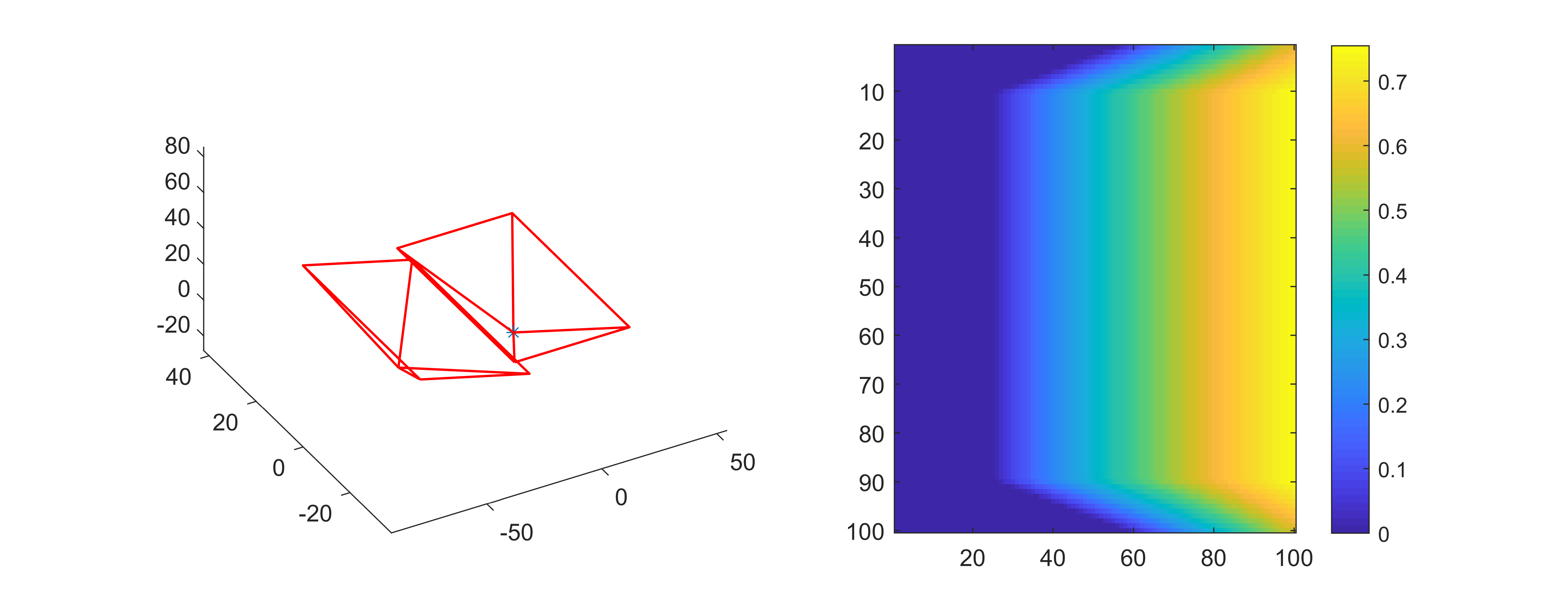}}
	\subfigure[Face-to-face cameras]{
		\includegraphics[width=9cm]{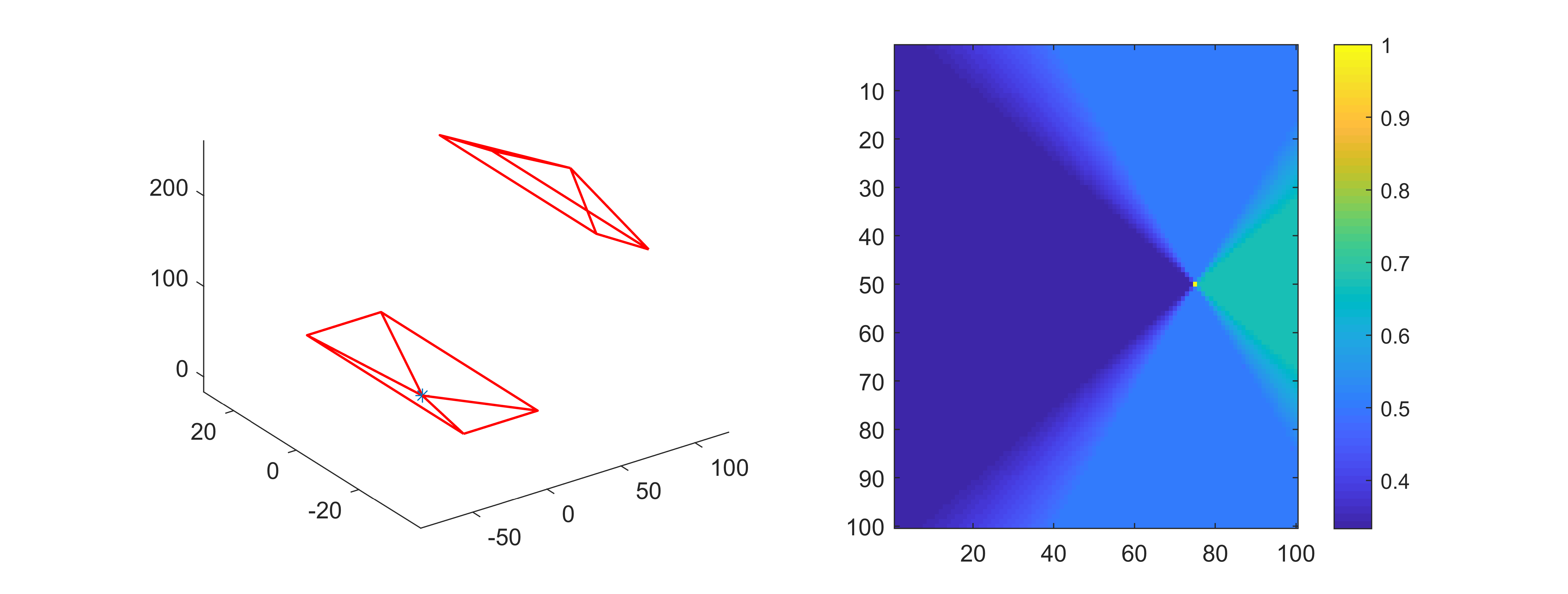}
	}
	\subfigure[Pure translation cameras]{
		\includegraphics[width=9cm]{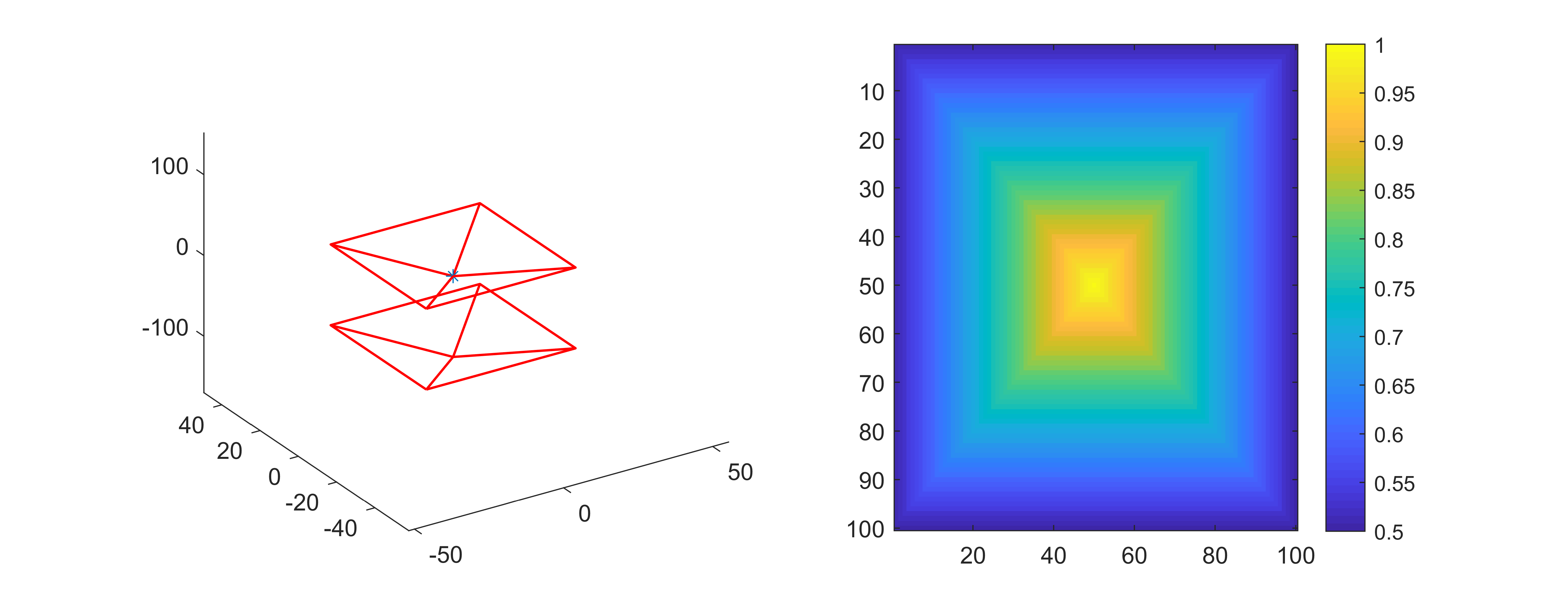}
	}
	\subfigure[Pure translation cameras2]{
		\includegraphics[width=9cm]{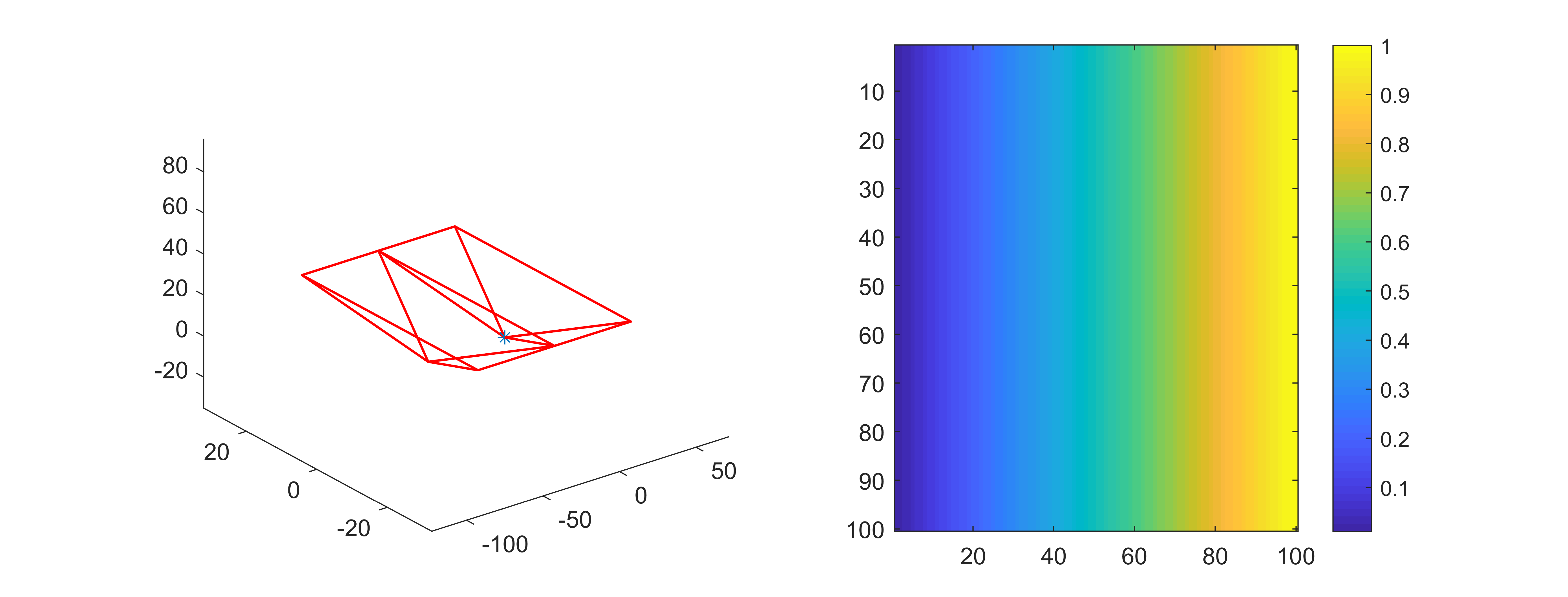}
	}
	\caption{4 types of relative pose of two view geometry and reduced search space for each image point}
	\label{simulatedcameras}
\end{figure}

The left subfigures in Fig.~\ref{simulatedcameras} demonstrates the relative pose of two camera. The optical center with a star is the first camera. The right subfigures show the reduced search space percentage for each image point by the advantage of identifying the right epipolar segment. Table~\ref{table1} lists the average search space reduced for four representative image pairs. On average, the right epipolar segment can help reduce the search space by about $50\%$ in contrast to the whole epipolar line.

\begin{table}
	\begin{center}
		\caption{REDUCED SEARCH SPACE FOR EACH IMAGE PAIR}
		\label{table1}
		\begin{tabular}{lll}
			\hline\noalign{\smallskip}
			pair & \quad reduced space(percentage) \\
			\noalign{\smallskip}
			\hline
			\noalign{\smallskip}
			\quad a  & \quad\quad\quad\quad\quad$31.07\%$\\
			\quad b & \quad\quad\quad\quad\quad$42.92\%$\\
			\quad c & \quad\quad\quad\quad\quad$66.67\%$\\
			\quad d & \quad\quad\quad\quad\quad$50.50\%$
			\\
			\hline
		\end{tabular}
	\end{center}
\end{table}
\setlength{\tabcolsep}{1.4pt}

\section{Conclusions}
This paper finds that for a known relative pose the epipolar line can be divided into two or three segments using the epipole and the newly-defined virtual infinity point. Hopefully, this fact can be used to significantly reduce the search space by about $50\%$ in finding feature correspondence. The epipolar line segment can also be used to identify the outliers of feature matching. Simulation and test results are reported to verify the analysis and the proposed algorithm. This paper discusses the problem of segmenting epipolar line in the calibrated case, the uncalibrated case would be discussed in the future work.

\bibliographystyle{ieeetr}
\bibliography{egbib}

\begin{thebibliography}{10}

\bibitem{lowe2004distinctive}
D.~G. Lowe, ``Distinctive image features from scale-invariant keypoints,'' {\em
  International Journal of Computer Vision}, vol.~60, no.~2, pp.~91--110, 2004.

\bibitem{bay2006surf}
H.~Bay, T.~Tuytelaars, and L.~Van~Gool, ``Surf: Speeded up robust features,''
  in {\em European Conference on Computer Vision}, pp.~404--417, Springer,
  2006.

\bibitem{hartley2003multiple}
R.~Hartley and A.~Zisserman, {\em Multiple view geometry in computer vision}.
\newblock Cambridge University Press, 2003.

\bibitem{nister2004efficient}
D.~Nist{\'e}r, ``An efficient solution to the five-point relative pose
  problem,'' {\em IEEE transactions on pattern analysis and machine
  intelligence}, vol.~26, no.~6, pp.~756--770, 2004.

\bibitem{stewenius2006recent}
H.~Stewenius, C.~Engels, and D.~Nist{\'e}r, ``Recent developments on direct
  relative orientation,'' {\em ISPRS Journal of Photogrammetry and Remote
  Sensing}, vol.~60, no.~4, pp.~284--294, 2006.

\bibitem{longuet1981computer}
H.~C. Longuet-Higgins, ``A computer algorithm for reconstructing a scene from
  two projections,'' {\em Nature}, vol.~293, no.~5828, pp.~133--135, 1981.

\bibitem{fusiello2000compact}
A.~Fusiello, E.~Trucco, and A.~Verri, ``A compact algorithm for rectification
  of stereo pairs,'' {\em Machine Vision and Applications}, vol.~12, no.~1,
  pp.~16--22, 2000.

\bibitem{oram2001rectification}
D.~Oram, ``Rectification for any epipolar geometry.,'' in {\em BMVC}, vol.~1,
  pp.~653--662, Citeseer, 2001.

\bibitem{pollefeys1999simple}
M.~Pollefeys, R.~Koch, and L.~Van~Gool, ``A simple and efficient rectification
  method for general motion,'' in {\em Proceedings of the Seventh IEEE
  International Conference on Computer Vision}, vol.~1, pp.~496--501, IEEE,
  1999.

\bibitem{qian1997binocular}
N.~Qian, ``Binocular disparity and the perception of depth,'' {\em Neuron},
  vol.~18, no.~3, pp.~359--368, 1997.

\bibitem{gonzalez1998neural}
F.~Gonzalez and R.~Perez, ``Neural mechanisms underlying stereoscopic vision,''
  {\em Progress in neurobiology}, vol.~55, no.~3, pp.~191--224, 1998.

\bibitem{blumenthal2014high}
D.~C. Blumenthal-Barby and P.~Eisert, ``High-resolution depth for binocular
  image-based modeling,'' {\em Computers \& Graphics}, vol.~39, pp.~89--100,
  2014.

\bibitem{10.1023/A:1007984508483}
R.~I. Hartley, ``Chirality,'' {\em International Journal of Computer Vision},
  vol.~26, p.~41–61, Jan. 1998.

\bibitem{werner2001cheirality}
T.~Werner and T.~Pajdla, ``Cheirality in epipolar geometry,'' in {\em
  Proceedings Eighth IEEE International Conference on Computer Vision. ICCV
  2001}, vol.~1, pp.~548--553, IEEE, 2001.

\bibitem{agarwal2020multiview}
S.~Agarwal, A.~Pryhuber, R.~Sinn, and R.~R. Thomas, ``Multiview chirality,''
  {\em arXiv preprint arXiv:2003.09265}, 2020.

\bibitem{wang2000svd}
W.~Wang and H.~T. Tsui, ``A svd decomposition of essential matrix with eight
  solutions for the relative positions of two perspective cameras,'' in {\em
  Proceedings 15th International Conference on Pattern Recognition. ICPR-2000},
  vol.~1, pp.~362--365, IEEE, 2000.

\end{thebibliography}

\end{document}